\documentclass[runningheads]{llncs}

\usepackage{eccv}

\usepackage{eccvabbrv}
\usepackage{arydshln}
\usepackage{graphicx}
\usepackage{booktabs}
\usepackage{enumerate}
\usepackage[accsupp]{axessibility}  

\usepackage{hyperref}

\usepackage{orcidlink}
\usepackage{thm-restate}

\begin{document}

\title{On the Asymptotic Redundancy of Timestep Conditioning in Diffusion Models} 
\titlerunning{On the Asymptotic Redundancy of Timestep Conditioning in Diffusion Models}
\author{José A. Chávez}

\authorrunning{José A. Chávez}
\institute{Independent Researcher, Lima, Peru \\
\email{josechavez.ie@gmail.com}} 

\maketitle

\begin{abstract}
Diffusion models typically rely on explicit timestep embeddings to condition the denoising network on the current noise level. In this work, we investigate whether such conditioning is necessary for optimal noise prediction. We develop a theoretical framework showing that, under suitable concentration assumptions on the clean-data energy and for a finite injective diffusion schedule, the noise scale can be consistently estimated from a corrupted sample as its dimensionality increases. This implies asymptotic equivalence, in normalized mean squared error, between the optimal time-dependent and time-agnostic noise predictors. We complement this analysis with an empirical study of U-Net and Diffusion Transformer architectures on CelebA and CIFAR-10, comparing models trained with and without explicit timestep embeddings. In exploratory experiments, removing timestep conditioning did not consistently degrade the metrics relative to our baselines. Because these baselines are far from published results, the comparison is inconclusive. Our work provides a theoretical characterization of when timestep information can become asymptotically redundant for optimal noise prediction.

\end{abstract}

\section{Introduction}
Diffusion Models~\cite{pmlr-v37-sohl-dickstein15} have shown significant success in generative modeling, synthesizing high-fidelity data through a gradual denoising process~\cite{DDPM_2020, Rombach_2022_CVPR,Blattmann_2023_CVPR,sauer2023adversarialdiffusiondistillation}. These models rely heavily on timestep embeddings, which serve to encode temporal information to adjust the noise removal level at each step. This information is typically injected via feature map addition or modulation across the model architecture.

Recently, transformer architectures have emerged as a powerful alternative to convolutional backbones in diffusion models. Diffusion Transformers (DiT)~\cite{Peebles_2023_ICCV}, based on the Vision Transformer architecture~\cite{dosovitskiy2021an}, provide several key advantages over traditional U-Net architectures. Chief among these is scalability; DiT demonstrates superior performance as model size and computational power increase, showing strong correlations between GFLOPs and sample quality. This allows for superior performance in high-resolution image generation and large-scale datasets. Furthermore, transformers excel at modeling long-range dependencies due to their self-attention mechanism, enabling better coherence in image generation by capturing relationships across distant regions of an image. DiTs outperformed prior U-Net-based diffusion models on class-conditional ImageNet~\cite{5206848} at $512\times512$ and $256\times256$ resolution~\cite{Peebles_2023_ICCV}.

\begin{figure}[h]
\centering
\renewcommand{\arraystretch}{0.5} 
\begin{tabular}{ll}
    \rotatebox{90}{ \small DiT} & 
    \includegraphics[width=0.92\columnwidth]{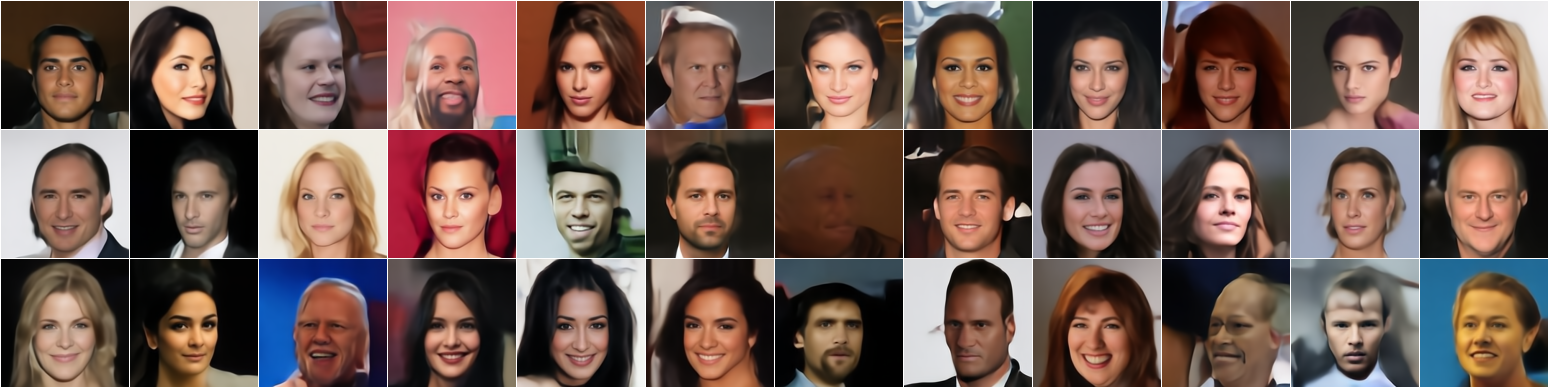} \\
    \noalign{\smallskip}\hdashline\noalign{\smallskip}
    
    \centering \rotatebox{90}{\small \quad DiT*} & 
    \includegraphics[width=0.92\columnwidth]{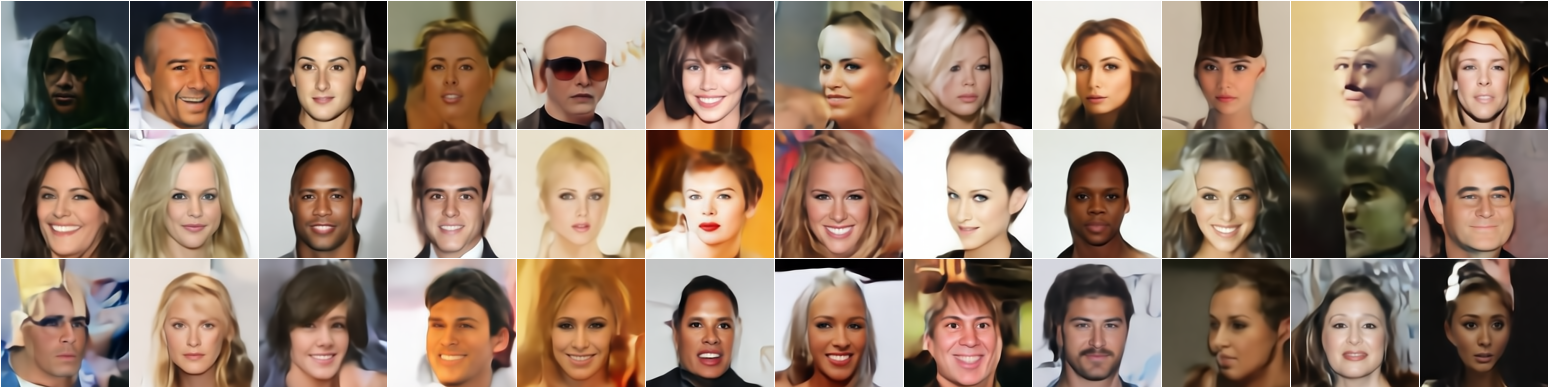} \\
    \noalign{\smallskip}\hdashline\noalign{\smallskip}
    
    \rotatebox{90}{\small \quad U-Net} & 
    \includegraphics[width=0.92\columnwidth]{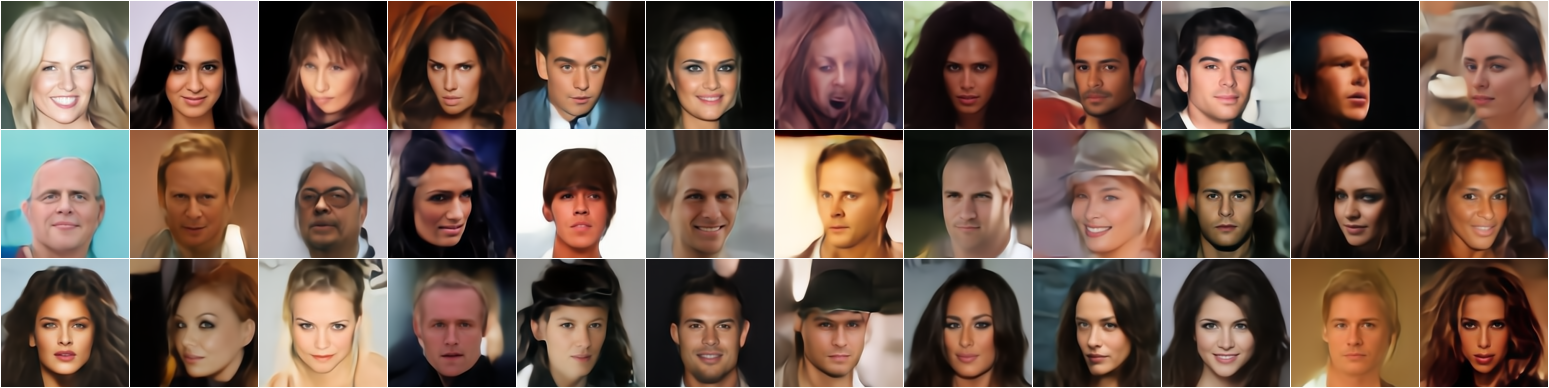} \\
    \noalign{\smallskip}\hdashline\noalign{\smallskip}
    
    \centering \rotatebox{90}{\small \quad U-Net*} & 
    \includegraphics[width=0.92\columnwidth]{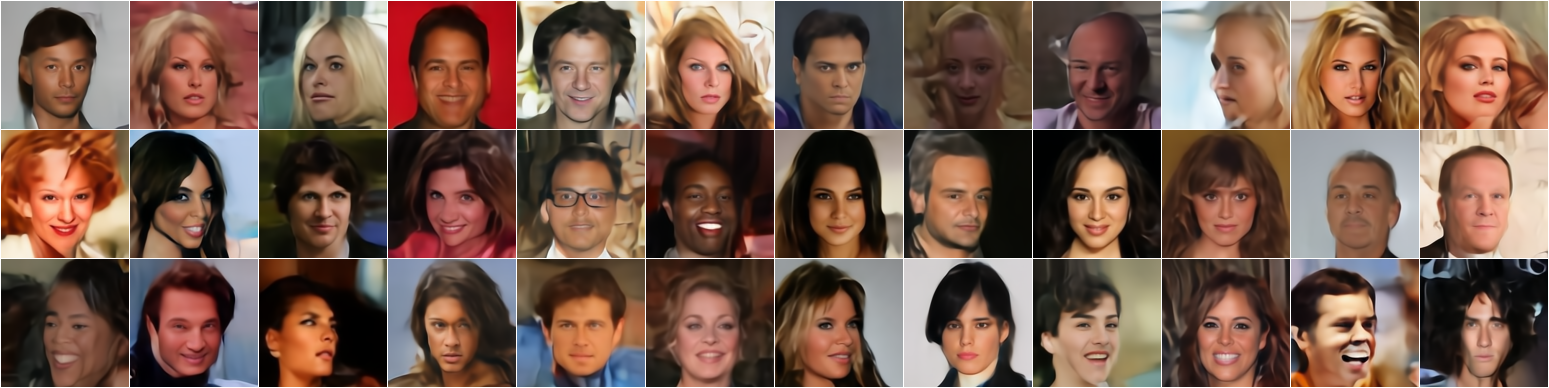} \\
\end{tabular}
\caption{\textbf{Unconditional face generation on the CelebA dataset ($128\times128$).} We compare the visual fidelity of standard architectures against our time-agnostic variants. The first and third rows show samples from standard DiT and U-Net backbones with explicit timestep conditioning. The second and fourth rows (marked with *) display samples from our time-agnostic counterparts.}
\label{fig:celeba_comparison}
\end{figure}

Regardless of the architectural backbone chosen---be it a U-Net or a DiT---current implementations fundamentally depend on a mechanism to inject the temporal signal or the noise scale. Although the importance of timestep embeddings is widely accepted, there has been limited investigation into their actual necessity. In this work, we investigate whether conditioning on the timestep provides asymptotically additional information for optimal noise prediction once the corrupted sample is observed. We then empirically examine whether this theoretical observation is reflected in finite-capacity U-Net and DiT models trained without explicit timestep embeddings.

\section{Background}

 Denoising diffusion probabilistic models (DDPM)~\cite{DDPM_2020} are latent variable models of the form $p_\theta(\mathbf{x}_0):=\int p_\theta(\mathbf{x}_{0:T})d\mathbf{x}_{1:T}$, where $\mathbf{x}_1,\dots,\mathbf{x}_{T}$ are 
 latent variables in the same sample space as $\mathbf{x}_{0}\sim q(\mathbf{x}_0)$. The \textit{forward process} is defined as a Markov chain that gradually adds Gaussian noise to the data:

\begin{align}
    \label{q_formula}
    q(\mathbf{x}_{1:T}\mid\mathbf{x}_0):=\prod^{T}_{t=1}q(\mathbf{x}_{t}\mid\mathbf{x}_{t-1}),
\end{align}

where the conditional forward distribution $q(\mathbf{x}_t\mid\mathbf{x}_0):=\mathcal{N}(\mathbf{x}_t;\sqrt{\bar{\alpha}_t}\mathbf{x}_0,(1-\bar{\alpha}_t)\mathbf{I})$ is parametrized by a time-dependent scalar $\bar{\alpha}_t\in\{\bar{\alpha}_1,\dots,\bar{\alpha}_T\}$. After a re-parameterization we can write:

\begin{align}
\label{x_t}
    \mathbf{x}_t:=\sqrt{\bar{\alpha}_t}\mathbf{x}_0+\sqrt{1-\bar{\alpha}_t}\boldsymbol{\epsilon}, 
\end{align}

where $\boldsymbol{\epsilon}\sim\mathcal{N}(\mathbf{0},\mathbf{I})$. The \textit{reverse process} is defined as a Markov chain:

 \begin{align}
     \label{p_formula}
     p_\theta(\mathbf{x}_{0:T}):=p_\theta(\mathbf{x}_T)\prod^{T}_{t=1}p_\theta(\mathbf{x}_{t-1}\mid\mathbf{x}_t).
 \end{align}

\section{The Redundancy of Timesteps in the Loss Function}

Given the forward diffusion process defined in \cref{x_t}, where $\mathbf{x}_0\sim q(\mathbf{x}_0)$ and
$\boldsymbol{\epsilon}\sim\mathcal{N}(\mathbf{0},\mathbf{I})$
are independent, the DDPM training objective can be written as
\begin{align}
    \label{l_theta_extended}
    \mathcal{L}(\theta)
    =
    \mathbb{E}_{\mathbf{x}_0,\boldsymbol{\epsilon},t}
    \left[
    \left\|
    \boldsymbol{\epsilon}
    -
    \boldsymbol{\epsilon}_{\theta}
    \left(
    \sqrt{\bar{\alpha}_t}\mathbf{x}_0
    +
    \sqrt{1-\bar{\alpha}_t}\boldsymbol{\epsilon},
    t
    \right)
    \right\|_2^2
    \right].
\end{align}

This objective explicitly conditions the noise predictor on the
diffusion timestep $t$. However, the timestep determines the
corruption level through the noise scale $\bar{\alpha}_t$.
This raises a fundamental question: is the explicit timestep
necessary if its associated noise scale can be inferred directly
from the corrupted sample?

We investigate this question by establishing three results.
First, we show that the noise scale can be consistently estimated
from a corrupted sample as its dimensionality increases.
Second, we use this result to establish the asymptotic
identifiability of the timestep under an injective diffusion
schedule. Finally, we show that this identifiability implies
asymptotic equivalence, in normalized mean squared error,
between the optimal time-dependent and time-agnostic noise
predictors.

\subsection{Asymptotic Noise Scale Estimation}

We first investigate whether the noise scale
$\bar{\alpha}_t$ can be inferred directly from the corrupted
sample. The key observation is that, under a suitable assumption
on the asymptotic energy of the clean data, the normalized
squared norm of the corrupted sample retains information about
the noise scale.

\begin{restatable}{theorem}{noisescalethm}
\label{theorem:noisescalethm}

Let $\mathbf{x}_t^{(d)}\in\mathbb{R}^d$ be a sequence of corrupted samples generated by the forward diffusion process

\[
\mathbf{x}_t^{(d)}
=
\sqrt{\bar{\alpha}_t}\mathbf{x}_0^{(d)}
+
\sqrt{1-\bar{\alpha}_t}\boldsymbol{\epsilon}^{(d)},
\]
where $\mathbf{x}_0^{(d)}$ and $\boldsymbol{\epsilon}^{(d)}$
are independent random vectors in $\mathbb{R}^d$, with
\[
\boldsymbol{\epsilon}^{(d)}
\sim\mathcal{N}(\mathbf{0},\mathbf{I}_d),
\qquad
\bar{\alpha}_t\in(0,1).
\]

Assume that, as $d\to\infty$,
\[
\frac{\|\mathbf{x}_0^{(d)}\|_2^2}{d}
\xrightarrow{P} C,
\]
where $C>0$ is a known constant satisfying $C\neq 1$.

Define the estimator
$$
\Phi_d(\mathbf{x}_t^{(d)})
=
\min\{1,\max\{0,\left(
\frac{
\|\mathbf{x}_t^{(d)}\|_2^2/d-1
}{C-1}
\right)\}\}.
$$

Then, for any fixed timestep $t$,
\[
\Phi_d(\mathbf{x}_t^{(d)})
\xrightarrow{P}\bar{\alpha}_t.
\]
    
\end{restatable}

Theorem~\ref{theorem:noisescalethm} establishes that the noise
scale can be consistently estimated from the corrupted sample
without explicitly observing the timestep. In particular, the
normalized squared norm satisfies
\[
\frac{\|\mathbf{x}_t^{(d)}\|_2^2}{d}
\xrightarrow{P}
1+\bar{\alpha}_t(C-1),
\]
which allows us to recover $\bar{\alpha}_t$ asymptotically.

This result provides the basis for identifying the timestep
itself, provided that different timesteps correspond to
different noise scales.

\subsection{Asymptotic Timestep Identifiability}

Consider a finite diffusion schedule
$\mathcal{T}=\{1,\ldots,T\}$, where $T$ is fixed, and assume
that the schedule is injective; that is,
$\bar{\alpha}_s\neq\bar{\alpha}_t$ whenever $s\neq t$.

Since the schedule is finite and injective, its minimum
separation is strictly positive:
\[
\Delta
=
\min_{\substack{s,t\in\mathcal{T}\\s\neq t}}
|\bar{\alpha}_s-\bar{\alpha}_t|
>0.
\]

Using the noise scale estimator from
Theorem~\ref{theorem:noisescalethm}, we define the timestep
estimator as the nearest schedule value:
\[
\widehat{t}_d(\mathbf{x}_t^{(d)})
=
\arg\min_{s\in\mathcal{T}}
\left|
\Phi_d(\mathbf{x}_t^{(d)})
-
\bar{\alpha}_s
\right|.
\]

The following theorem establishes that this estimator
consistently recovers the timestep.

\begin{restatable}{theorem}{timestepidentifiabilitythm}
\label{theorem:timestep_identifiability}

Under the assumptions of
Theorem~\ref{theorem:noisescalethm}, assume additionally
that the diffusion schedule is injective over a finite
timestep set $\mathcal{T}$.

Define the timestep estimator
\[
\widehat{t}_d(\mathbf{x}_t^{(d)})
=
\arg\min_{s\in\mathcal{T}}
\left|
\Phi_d(\mathbf{x}_t^{(d)})
-
\bar{\alpha}_s
\right|.
\]

Then, for any fixed $t\in\mathcal{T}$,
\[
\lim_{d\to\infty}
P\left(
\widehat{t}_d(\mathbf{x}_t^{(d)})=t
\right)
=1.
\]

Therefore, the diffusion timestep is asymptotically
identifiable from the corrupted sample.

\end{restatable}

Theorem~\ref{theorem:timestep_identifiability} follows directly
from Theorem~\ref{theorem:noisescalethm}. Indeed, whenever
\[
\left|
\Phi_d(\mathbf{x}_t^{(d)})-\bar{\alpha}_t
\right|
<
\frac{\Delta}{2},
\]
the true noise scale is the unique closest value in the
diffusion schedule. Since the noise scale estimator is
consistent, the probability of recovering the correct
timestep converges to one.

Thus, the timestep can be asymptotically recovered from the
corrupted sample, without requiring it as an explicit input.

The finite and injective schedule assumption is essential to the nearest-neighbor identification argument, since it ensures a strictly positive minimum separation between distinct noise scales.

\subsection{Asymptotic Equivalence of Optimal Time-Dependent and Time-Agnostic Predictors}

We now investigate the implications of timestep identifiability
for optimal noise prediction. In particular, we compare the
optimal time-dependent predictor, which observes both the
corrupted sample and its timestep, with the optimal
time-agnostic predictor, which only observes the corrupted
sample.

For notational convenience, let
$\mathbf{Z}=\mathbf{x}_t^{(d)}$ denote the random
corrupted sample. Define the two optimal predictors as
\[
\boldsymbol{\eta}_{TD}^{(d)}
=
\mathbb{E}\left[
\boldsymbol{\epsilon}^{(d)}
\mid \mathbf{Z},t
\right],
\qquad
\boldsymbol{\eta}_{TA}^{(d)}
=
\mathbb{E}\left[
\boldsymbol{\epsilon}^{(d)}
\mid \mathbf{Z}
\right].
\]

Their corresponding risks are
\[
R_{TD}^{(d)}
=
\mathbb{E}\left[
\left\|
\boldsymbol{\epsilon}^{(d)}
-
\boldsymbol{\eta}_{TD}^{(d)}
\right\|_2^2
\right],
\]
and
\[
R_{TA}^{(d)}
=
\mathbb{E}\left[
\left\|
\boldsymbol{\epsilon}^{(d)}
-
\boldsymbol{\eta}_{TA}^{(d)}
\right\|_2^2
\right].
\]

All expectations above are taken with respect to the joint
distribution of the timestep, the clean sample, and the
Gaussian noise.

By the orthogonal projection property of conditional
expectation, their risk difference satisfies
\[
R_{TA}^{(d)}-R_{TD}^{(d)}
=
\mathbb{E}\left[
\left\|
\boldsymbol{\eta}_{TD}^{(d)}
-
\boldsymbol{\eta}_{TA}^{(d)}
\right\|_2^2
\right].
\]

This identity shows that the difference between the two
optimal risks is determined by how much the time-dependent
predictor varies with the timestep when the corrupted sample
is fixed. We can therefore use the asymptotic identifiability
established in Theorem~\ref{theorem:timestep_identifiability}
to bound this difference.

\begin{restatable}{theorem}{predictorequivalencethm}
\label{theorem:predictor_equivalence}

Assume the conditions of
Theorems~\ref{theorem:noisescalethm} and
\ref{theorem:timestep_identifiability}. Let $T$ denote the total number of diffusion steps, and
let $\mathcal{T}=\{1,\ldots,T\}$ be the corresponding
finite set of timestep indices.

During training, let $t\sim\operatorname{Unif}(\mathcal{T})$, where
$t$, $\mathbf{x}_0^{(d)}$ and $\boldsymbol{\epsilon}^{(d)}$ are
mutually independent, and let the schedule and $T$ not depend on $d$.

Consider the forward diffusion process
\[
\mathbf{x}_{t}^{(d)}
=
\sqrt{\bar{\alpha}_{t}}\mathbf{x}_0^{(d)}
+
\sqrt{1-\bar{\alpha}_{t}}
\boldsymbol{\epsilon}^{(d)},
\]
where
\[
\boldsymbol{\epsilon}^{(d)}
\sim\mathcal{N}(\mathbf{0},\mathbf{I}_d).
\]

Define the optimal time-dependent and time-agnostic
noise predictors as
\[
\boldsymbol{\eta}_{TD}^{(d)}
=
\mathbb{E}\left[
\boldsymbol{\epsilon}^{(d)}
\mid \mathbf{x}_{t}^{(d)},t
\right],
\]
and
\[
\boldsymbol{\eta}_{TA}^{(d)}
=
\mathbb{E}\left[
\boldsymbol{\epsilon}^{(d)}
\mid \mathbf{x}_{t}^{(d)}
\right].
\]

Their respective mean squared errors are
\[
R_{TD}^{(d)}
=
\mathbb{E}\left[
\left\|
\boldsymbol{\epsilon}^{(d)}
-
\boldsymbol{\eta}_{TD}^{(d)}
\right\|_2^2
\right],
\]
and
\[
R_{TA}^{(d)}
=
\mathbb{E}\left[
\left\|
\boldsymbol{\epsilon}^{(d)}
-
\boldsymbol{\eta}_{TA}^{(d)}
\right\|_2^2
\right].
\]

Then, with $p_d=\mathbb{P}\big(\widehat{t}_d(\mathbf{x}_t^{(d)})
\neq t\big)$,

\[
0\le\frac{R_{TA}^{(d)}-R_{TD}^{(d)}}{d}
\le\min\Big\{1,\;2\sqrt{2\big(1+\tfrac{2}{d}\big)p_d}\Big\}.
\]

Consequently,
\[
\boxed{
\lim_{d\to\infty}
\frac{
R_{TA}^{(d)}-R_{TD}^{(d)}
}{d}
=0.
}
\]

Therefore, the time-agnostic predictor asymptotically
achieves the same normalized mean squared error as
the optimal time-dependent predictor.

\end{restatable}

\cref{theorem:predictor_equivalence} establishes the
main theoretical result of this section. By
\cref{theorem:timestep_identifiability}, the probability
of incorrectly estimating the timestep converges to zero.
Since the timestep is sampled uniformly from a finite set,
we have
\[
P\left(
\widehat{t}_d(\mathbf{x}_t^{(d)})\neq t
\right)
=
\frac{1}{T}
\sum_{t=1}^{T}
P\left(
\widehat{t}_d(\mathbf{x}_t^{(d)})\neq t
\mid t
\right)
\xrightarrow[d\to\infty]{}0.
\]

Substituting this result into the bound of
Theorem~\ref{theorem:predictor_equivalence} yields the
vanishing normalized risk difference.

Consequently, although the time-dependent predictor has access to the timestep explicitly, this additional information becomes asymptotically redundant for optimal noise prediction under the stated assumptions.

The equivalence established above concerns the population-optimal predictors and their normalized mean squared errors. It does not imply pointwise equality of the predictors, exact equality of finite-dimensional risks, or equivalent performance for neural networks with finite capacity and imperfect optimization. The theoretical result instead characterizes the asymptotic value of timestep information under the stated assumptions.

\section{Conditions on the Clean-Data Distribution and Finite-Dimensional Implications}

We discuss the assumptions on the clean-data distribution
and the implications of our theoretical results for
finite-dimensional image data.

\subsection{Concentration of the clean-data energy.}

Consider a sequence of clean image distributions
$\{q_d\}_{d\geq1}$, where each image
$\mathbf{x}_0^{(d)}\in\mathbb{R}^d$ has bounded components:
\[
|x_{0,i}^{(d)}|\leq m
\quad\text{almost surely},
\]
for some constant $m>0$ independent of $d$.

Define the normalized expected energy as
\[
C_d
:=
\mathbb{E}\left[
\frac{\|\mathbf{x}_0^{(d)}\|_2^2}{d}
\right]
=
\frac{1}{d}\sum_{i=1}^{d}
\mathbb{E}\left[(x_{0,i}^{(d)})^2\right].
\]

Assume that
\[
C_d\longrightarrow C
\]
for some constant $C>0$, and that the squared components
satisfy the following weak-dependence condition:
\[
\sum_{i<j}
\left|
\operatorname{Cov}
\left(
(x_{0,i}^{(d)})^2,
(x_{0,j}^{(d)})^2
\right)
\right|
=O(d).
\]

Since the components are uniformly bounded, their fourth
moments are also uniformly bounded. Consequently,
\[
\begin{aligned}
\operatorname{Var}\left(
\|\mathbf{x}_0^{(d)}\|_2^2
\right)
&=
\sum_{i=1}^{d}
\operatorname{Var}\left((x_{0,i}^{(d)})^2\right)\\
&\quad+
2\sum_{i<j}
\operatorname{Cov}\left(
(x_{0,i}^{(d)})^2,
(x_{0,j}^{(d)})^2
\right)\\
&=O(d).
\end{aligned}
\]

Therefore, Chebyshev's inequality gives, for any
$\delta>0$,
\[
\begin{aligned}
P\left(
\left|
\frac{\|\mathbf{x}_0^{(d)}\|_2^2}{d}
-C
\right|>\delta
\right)
&\leq
P\left(
\left|
\frac{\|\mathbf{x}_0^{(d)}\|_2^2}{d}
-C_d
\right|>\frac{\delta}{2}
\right)\\
&\quad+
P\left(
|C_d-C|>\frac{\delta}{2}
\right)
\longrightarrow0.
\end{aligned}
\]

Hence,
\[
\boxed{
\frac{\|\mathbf{x}_0^{(d)}\|_2^2}{d}
\xrightarrow{P}C.
}
\]

This establishes the concentration assumption required
by Theorem~\ref{theorem:noisescalethm} under the stated
boundedness and weak-dependence conditions.

\subsection{Implications for image and latent-space representations.}

The assumptions on the clean-data distribution depend on the representation used by the forward diffusion process. In our experiments, CIFAR-10 images are normalized to $[-1,1]$ and used directly as diffusion inputs. Their components are therefore uniformly bounded, and their normalized expected energy satisfies
\[
C_{\mathrm{CIFAR}} = \mathbb{E}\left[\frac{\|\mathbf{x}_0\|_2^2}{d}\right] \le 1,
\]
with $0 < C_{\mathrm{CIFAR}} < 1$ in practice, empirically estimated at $C_{\mathrm{CIFAR}} \approx 0.256$.

For CelebA, input images are normalized to $[-1,1]$ and encoded using a pretrained VAE. Dividing the resulting latent representation $\mathbf{z}_0$ by a scaling factor of $25$ before the forward diffusion process yields
\[
\mathbf{x}_0 = \frac{\mathbf{z}_0}{25}, \qquad C_{\mathrm{CelebA}} = \frac{1}{625}\,\mathbb{E}\left[\frac{\|\mathbf{z}_0\|_2^2}{d}\right],
\]

with an empirical value of $C_{\mathrm{CelebA}} \approx 0.082$. Both values are away from one, as our estimator requires. They do not by themselves establish concentration of $\|\mathbf{x}_0\|_2^2/d$, which we state as an assumption; it may fail for natural images, whose global brightness and contrast induce long-range correlations. Moreover, VAE latents need not be uniformly bounded.

\section{Related Work}

Diffusion models~\cite{pmlr-v37-sohl-dickstein15} adopt an idea from physics, where a sample is gradually converted into a well-known distribution (e.g., Gaussian) by adding noise in a forward process. Then, a reverse process is learned to restore the original data from the noise. In DDPM~\cite{DDPM_2020}, the reverse process is parameterized by a neural network $\boldsymbol{\epsilon}_\theta$ that takes as input the noisy image $\mathbf{x}_t$ and the temporal signal $t$ to predict the noise $\boldsymbol{\epsilon}$ used to corrupt the original image $\mathbf{x}_0$. To inject the scalar $t$ into the model, DDPM uses positional encodings followed by a feedforward network. In contrast, DiT~\cite{Peebles_2023_ICCV} uses adaptive layer normalization (adaLN-Zero) to inject the time signal into the model.

Although recent works have explored noise-agnostic and blind diffusion models, the relationship between timestep identifiability from a single corrupted sample and the asymptotic optimal noise-prediction risk remains an interesting theoretical question. We investigate this relationship by establishing sufficient conditions under which the timestep becomes asymptotically identifiable from the corrupted sample and the normalized risk gap between time-dependent and time-agnostic optimal predictors vanishes.

\begin{figure}[!h]
\centering
\renewcommand{\arraystretch}{0.5}
\begin{tabular}{ll}
    \rotatebox{90}{ \small DiT} & 
    \includegraphics[width=0.92\columnwidth]{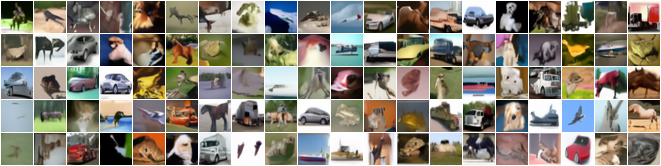} \\
    \noalign{\smallskip}\hdashline\noalign{\smallskip}
    
    \centering \rotatebox{90}{\small \quad DiT*} & 
    \includegraphics[width=0.92\columnwidth]{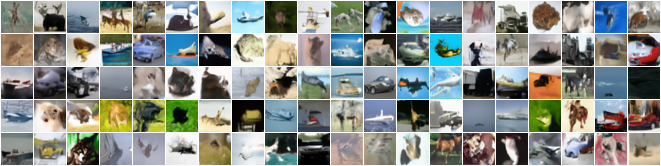} \\
    \noalign{\smallskip}\hdashline\noalign{\smallskip}
    
    \rotatebox{90}{\small \quad U-Net} & 
    \includegraphics[width=0.92\columnwidth]{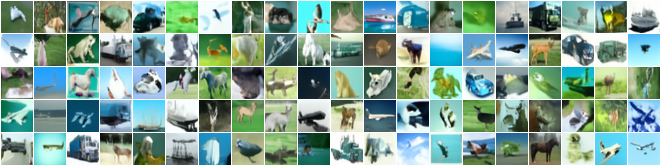} \\
    \noalign{\smallskip}\hdashline\noalign{\smallskip}
    
    \centering \rotatebox{90}{\small \quad U-Net*} & 
    \includegraphics[width=0.92\columnwidth]{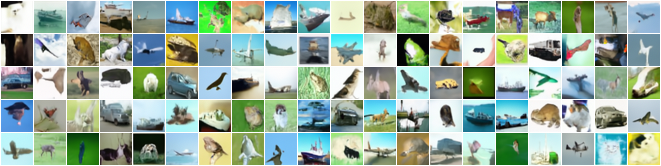} \\
\end{tabular}
\caption{\textbf{Qualitative comparison of conditional-class generation on CIFAR-10.} The first and third rows show samples from standard DiT and U-Net backbones with explicit timestep conditioning. The second and fourth rows (marked with *) display samples from our time-agnostic counterparts.}
\label{fig:cifar_comparison}
\end{figure}

Denoising Autoencoders (DAE)~\cite{denoising_auto_2008} and Generalized DAEs (GDAE)~\cite{10.5555/2999611.2999712} established early on that generative density modeling can be achieved through iterative corruption-reversal processes without strictly requiring explicit scale or corruption inputs, as exemplified by the Walkback algorithm. Furthermore, Denoising Score Matching (DSM)~\cite{vincent_2011} proved that the DAE objective is equivalent to estimating the score function $\nabla_{\mathbf{x}_t}\log p(\mathbf{x}_t)$ under a specific noise distribution. This fundamental connection directly underpins Noise Conditional Score Networks~\cite{NEURIPS2019_yang}.

Although the subsequent diffusion paradigm heavily adopted explicit time-conditioning, recent literature has increasingly explored modern noise-agnostic or blind denoising frameworks. For instance, Sun et al.~\cite{sun2025noise} study several formulations (iDDPM, DDIM, EDM and Flow Matching) and bound the effect of removing noise conditioning; part of their analysis relies on simplifying assumptions, such as a single data point.

In \cite{sahraee2026geometry}, researchers argue that blind noise-prediction models are structurally unstable, while velocity-based ones are stable. Our risk result does not address sampling stability and therefore does not contradict or establish their conclusions. Concurrently, \cite{kadkhodaie2026blind} provide theoretical guarantees for blind denoising diffusion models (BDDMs) based on the blessings of dimensionality, assuming that the data distribution has sufficiently low intrinsic dimensionality relative to the ambient dimension. However, these guarantees rely on the geometric structure of the data support, which may change under the nonlinear transformations induced by latent representations, as in Latent Diffusion Models (LDMs). Consequently, whether the low-intrinsic-dimensionality assumptions remain valid in latent spaces requires further investigation.

\section{Experiments}

\begin{table}[tb]
\caption{Evaluation of unconditional image generation on the CelebA dataset $128\times128$~\cite{Liu_2015_ICCV}. Numbers are reported for 50000 samples. *: models trained without the timestep embeddings.} 
\centering 
\begin{tabular}{@{}l@{\hskip 0.2in}c@{\hskip 0.2in}c@{\hskip 0.2in}c@{\hskip 0.2in}c@{}} 

\toprule 
Model & FID$\downarrow$ & Precision$\uparrow$ & Recall$\uparrow$ & Time (s)$\downarrow$\\ 
\midrule 
DiT & $\textbf{70.04}$ & $\textbf{0.68}$ & $0.13$ & $1.558\pm 0.046$\\ 
DiT* & $70.48$ & $0.65$ & $\textbf{0.14}$ & $\textbf{1.324}\pm 0.042$ \\
\midrule
U-Net & $69.82$ & $0.727$ & $0.10$ & $1.115\pm 0.047$\\
U-Net* & $\textbf{63.63}$ & $\textbf{0.729}$ & $\textbf{0.12}$ & $\textbf{1.089}\pm 0.042$\\
\bottomrule 
\end{tabular}
\label{table:celebametrics} 
\end{table}

\begin{table}[tb]
\caption{Evaluation of conditional-class image generation on CIFAR-10. Numbers are reported for 50000 samples. *: models trained without the timestep embeddings.} 
\centering 
\begin{tabular}{@{}l@{\hskip 0.2in}c@{\hskip 0.2in}c@{\hskip 0.2in}c@{\hskip 0.2in}c@{}} 

\toprule 
Model & FID$\downarrow$ & Precision$\uparrow$ & Recall$\uparrow$ & Time (s)$\downarrow$\\ 
\midrule 
DiT & $51.32$ & $0.68$ & $0.47$ & $8.553\pm 0.222$\\ 
DiT* & $\textbf{34.74}$ & $\textbf{0.75}$ & $\textbf{0.49}$ & $\textbf{8.498}\pm 0.420$\\
\midrule %
U-Net & $52.67$ & $0.66$ & $0.45$ & $7.105\pm  0.400$\\
U-Net* & $\textbf{41.18}$ & $ \textbf{0.68}$ & $\textbf{0.57}$ & $\textbf{7.048}\pm 0.368$\\
\bottomrule 
\end{tabular}
\label{table:cifar10metrics} 
\end{table}

In this section, we conduct an analysis of timestep embeddings in both U-Net and DiT architectures. We evaluate the impact of removing timestep embeddings on the quality of generated samples using standard metrics such as FID, Precision, and Recall. We also provide a qualitative comparison of the generated images to visually assess the differences in fidelity and diversity between models with and without timestep conditioning.

\subsubsection{Experimental Setup}

We adopt the common setup for measuring the Fréchet Inception Distance (FID)~\cite{housel_2017_fid}, following the standard procedure in the literature~\cite{DDPM_2020,karras_2019_stylegan}. For sample diversity and fidelity, we also report Precision and Recall metrics~\cite{NEURIPS2019_0234c510}. The setup for Precision and Recall is the same as in~\cite{NEURIPS2019_0234c510}. We train on the CelebA dataset with samples at $128\times128$ resolution, and we train on the CIFAR-10 dataset with samples at $32\times32$ resolution. For each dataset, we train two configurations for each architecture: one with timestep embeddings and another without them. For sampling time, we use the clock time when generating $100$ batches of $100$ samples each, and report the average time per batch along with the standard deviation across batches on an RTX 4070 GPU.

\subsubsection{Training Details}
The U-Net architecture is based on DDPM~\cite{DDPM_2020}, while the DiT architecture follows the DiT-S/2 configuration of~\cite{Peebles_2023_ICCV} with $252$ hidden units. For the CelebA dataset, we first train a VAE~\cite{kingma2013auto} to learn a lower-dimensional latent representation of the data. Then we train both U-Net and DiT in this latent space. For the CIFAR-10 dataset, we directly train both architectures in the data space. The diffusion model in CelebA is conducted in the latent space of the VAE, at $16\times16$ resolution. While in CIFAR-10, it is performed in the data space, at $32\times32$ resolution.

All diffusion models in our experiments were trained for 100 epochs. This training budget was constrained by the available computational resources (a single NVIDIA RTX 4070 GPU) and is substantially smaller than the training budgets commonly used in the reference implementations. Consequently, our time-conditioned models do not reach the performance reported by competitive published baselines. We therefore treat the experiments as preliminary ablations rather than controlled comparisons of fully trained models.

\subsection{Quantitative Results}

The quantitative results are summarized in \cref{table:celebametrics} and \cref{table:cifar10metrics}. In the CelebA dataset, DiT* (DiT without timestep embeddings) achieves an FID of $70.48$, which is slightly higher than the $70.04$ obtained by standard DiT. In terms of Precision and Recall, DiT* shows a minor decrease in Precision but an improvement in Recall. For U-Net, removal of timestep embeddings a lower FID from $69.82$ to $63.63$, as well as increases in both Precision and Recall.

The observed differences between the time-conditioned and time-agnostic configurations are larger on CIFAR-10 than on CelebA. However, given the limited training budget, these differences cannot be attributed reliably to the removal of timestep conditioning. DiT* obtains a lower FID from $51.32$ to $34.74$, along with increases in both Precision and Recall. Similarly, U-Net* obtains a lower FID from $52.67$ to $41.18$, Precision from $0.66$ to $0.68$, and Recall from $0.45$ to $0.57$. 

Within our preliminary experimental setup, the time-agnostic configurations obtain comparable or better reported metrics in three of the four comparisons. However, given the limited training budget and the absence of multiple training seeds, these observations should be interpreted as exploratory rather than conclusive evidence.
\subsection{Qualitative Analysis}

The qualitative comparison in \cref{fig:celeba_comparison} reveals that for the CelebA dataset, both DiT and U-Net generate visually coherent faces with and without timestep embeddings. However, samples from default DiT and U-Net (with timestep embeddings) exhibit fewer artifacts. In contrast, DiT* and U-Net* show a greater diversity in terms of facial features, color distribution, and pose, which may contribute to improved FID scores. In the CIFAR-10 dataset, the differences in \cref{fig:cifar_comparison} are more difficult to discern visually, but the samples of DiT* and U-Net* appear to have more varied textures and colors. These observations are subjective.

\subsection{Robustness and Performance Trade-offs}

Removing the timestep embedding clearly reduced the batch sampling time only for DiT on CelebA ($1.558\pm0.046$\,s to $1.324\pm0.042$\,s, about $15\%$). In the other three comparisons, the differences ($0.6\%$--$2.3\%$) are within one standard deviation.

\section{Discussion}

Our experiments show that time-agnostic variants can be competitive with, or better than, our time-conditioned baselines. We cannot conclude that removing timestep embeddings improves generation, as the absolute FID values of our baselines are far from published results. Undertraining, differences in model capacity, and sampler interactions are possible explanations that require ablations.

\cref{theorem:noisescalethm,theorem:timestep_identifiability,theorem:predictor_equivalence} describe a high-dimensional limit for optimal predictors. They do not explain the finite-dimensional results above, and we did not measure noise-level identifiability or per-timestep risk. Exact recovery of the timestep is also stronger than necessary; concentration of the posterior over a continuous noise level is a more natural condition.

\section{Conclusion}

In this work, we investigated the redundancy of explicit timestep conditioning in diffusion noise prediction. We showed that, under suitable concentration assumptions on the clean-data energy and a finite injective diffusion schedule, the noise scale can be consistently estimated from a corrupted sample using its normalized squared norm. This yields asymptotic timestep identifiability and establishes that the optimal time-agnostic predictor achieves the same normalized mean squared error as the optimal time-dependent predictor in the high-dimensional limit.

We complemented these theoretical results with experiments on U-Net and DiT architectures using CelebA and CIFAR-10. These exploratory results do not show a consistent degradation when explicit timestep embeddings are removed, but they are inconclusive given the quality of our baselines.

Our theoretical guarantees concern population-optimal predictors and asymptotic normalized risks; they do not establish equivalent performance for finite neural networks or finite-dimensional data. Further work is needed to empirically validate noise-scale identifiability, investigate the role of latent representations, and understand how architectural inductive biases affect time-agnostic denoising.

\section*{Limitations and Open Problems}

An important limitation is that our experiments do not isolate the effect of removing timestep conditioning from the effects of limited optimization. Since all models were trained under a restricted computational budget, the observed differences may reflect optimization dynamics or interactions between the architecture and the training schedule rather than the ability to infer the noise scale from the corrupted input.

Our experiments are exploratory: they were run on an RTX 4070 GPU,
and the absolute FID values of our baselines are far from published
results. The observed differences should therefore not be taken as
evidence that removing timestep conditioning improves generation. We leave the following problems open:
\begin{itemize}
\item reproducing the comparison with competitive baselines and
several training seeds;
\item measuring how accurately the noise scale can be inferred from
noisy samples, per timestep and per dimension;
\item relaxing the energy-concentration assumption for natural images
and latent representations;
\item relating the risk gap to sample quality along the reverse process.
\end{itemize}

\bibliographystyle{splncs04}
\bibliography{main}

\appendix
\section{Proofs}

\noisescalethm*

\begin{proof}
Consider the normalized squared norm of the corrupted sample:
\[
H_d(\mathbf{x}_t)=\frac{\|\mathbf{x}_t\|_2^2}{d}.
\]

Expanding the squared norm gives
\[
\begin{aligned}
H_d(\mathbf{x}_t)
={}&
\bar{\alpha}_t
\frac{\|\mathbf{x}_0\|_2^2}{d}
+
(1-\bar{\alpha}_t)
\frac{\|\boldsymbol{\epsilon}\|_2^2}{d}
\\
&+
2\sqrt{\bar{\alpha}_t(1-\bar{\alpha}_t)}
\frac{\mathbf{x}_0^T\boldsymbol{\epsilon}}{d}.
\end{aligned}
\]

By assumption,
\[
\frac{\|\mathbf{x}_0\|_2^2}{d}
\xrightarrow{P} C.
\]

Furthermore, since the components of
$\boldsymbol{\epsilon}$ are independent standard Gaussian
random variables, the law of large numbers implies
\[
\frac{\|\boldsymbol{\epsilon}\|_2^2}{d}
\xrightarrow{P}1.
\]

For the cross term, conditioning on $\mathbf{x}_0$ yields
\[
\mathbb{E}\left[
\left.
\left(
\frac{\mathbf{x}_0^T\boldsymbol{\epsilon}}{d}
\right)^2
\right|\mathbf{x}_0
\right]
=
\frac{\|\mathbf{x}_0\|_2^2}{d^2}.
\]

Hence, by Chebyshev's inequality, for every $\delta>0$,
\[
P\left(\left.\left|\frac{\mathbf{x}_0^T\boldsymbol{\epsilon}}{d}\right|>\delta\,\right|\mathbf{x}_0\right)
\leq\min\left\{1,\frac{\|\mathbf{x}_0\|_2^2}{d^2\delta^2}\right\}.
\]
Since $\|\mathbf{x}_0\|_2^2/d\xrightarrow{P}C$, the right-hand side
converges to zero in probability and is bounded by one. Taking
expectations and applying bounded convergence gives
\[
\frac{\mathbf{x}_0^T\boldsymbol{\epsilon}}{d}\xrightarrow{P}0.
\]

Combining these convergence results, we obtain
\[
H_d(\mathbf{x}_t)
\xrightarrow{P}
\bar{\alpha}_t C+(1-\bar{\alpha}_t)
=
1+\bar{\alpha}_t(C-1).
\]

Since $C\neq 1$, the continuous mapping theorem gives
\[
\frac{H_d(\mathbf{x}_t)-1}{C-1}
\xrightarrow{P}\bar{\alpha}_t.
\]

Finally, the clipping function
\[
g(u)=\min\{1,\max\{0,u\}\}
\]
is continuous and satisfies $g(\bar{\alpha}_t)=\bar{\alpha}_t$,
because $\bar{\alpha}_t\in(0,1)$.

Applying the continuous mapping theorem once more,
\[
\Phi_d(\mathbf{x}_t)
\xrightarrow{P}\bar{\alpha}_t.
\]

This completes the proof.
\end{proof}

\timestepidentifiabilitythm*

\begin{proof}
By Theorem~\ref{theorem:noisescalethm}, for any fixed
$t\in\mathcal{T}$,
\[
\Phi_d(\mathbf{x}_t^{(d)})
\xrightarrow{P}\bar{\alpha}_t.
\]

Since $\mathcal{T}$ is finite and the diffusion schedule
is injective, the minimum distance between distinct
noise scales is strictly positive. Define
\[
\Delta
=
\min_{\substack{s,t\in\mathcal{T}\\s\neq t}}
|\bar{\alpha}_s-\bar{\alpha}_t|>0.
\]

If
\[
\left|
\Phi_d(\mathbf{x}_t^{(d)})
-\bar{\alpha}_t
\right|<\frac{\Delta}{2},
\]
then $\bar{\alpha}_t$ is the unique closest noise scale
to $\Phi_d(\mathbf{x}_t^{(d)})$. Consequently,
\[
\widehat{t}_d(\mathbf{x}_t^{(d)})=t.
\]

Thus,
\[
\begin{aligned}
&P\left(
\widehat{t}_d(\mathbf{x}_t^{(d)})\neq t
\right)
\\
&\quad\leq
P\left(
\left|
\Phi_d(\mathbf{x}_t^{(d)})
-\bar{\alpha}_t
\right|
\geq\frac{\Delta}{2}
\right)
\longrightarrow 0,
\end{aligned}
\]
where the convergence follows from Theorem~\ref{theorem:noisescalethm}.

Therefore,
\[
P\left(
\widehat{t}_d(\mathbf{x}_t^{(d)})=t
\right)
\longrightarrow 1.
\]
\end{proof}

\predictorequivalencethm*

\begin{proof}

Write $\mathbf{Z}=\mathbf{x}_t^{(d)}$,
$\boldsymbol{\epsilon}=\boldsymbol{\epsilon}^{(d)}$,
$\boldsymbol{\eta}(\mathbf{Z},t)=
\mathbb{E}[\boldsymbol{\epsilon}\mid\mathbf{Z},t]$ and
$\overline{\boldsymbol{\eta}}(\mathbf{Z})=
\mathbb{E}[\boldsymbol{\epsilon}\mid\mathbf{Z}]$.
By the tower property,
$\overline{\boldsymbol{\eta}}(\mathbf{Z})=
\mathbb{E}[\boldsymbol{\eta}(\mathbf{Z},t)\mid\mathbf{Z}]$, and by
orthogonality of conditional expectation,
\begin{equation}\label{eq:gap}
R_{TA}^{(d)}-R_{TD}^{(d)}
=\mathbb{E}\|\boldsymbol{\eta}(\mathbf{Z},t)
-\overline{\boldsymbol{\eta}}(\mathbf{Z})\|_2^2 .
\end{equation}
In particular, the gap is nonnegative. It is at most $d$, since
$R_{TD}^{(d)}\ge 0$ and
$R_{TA}^{(d)}\le\mathbb{E}\|\boldsymbol{\epsilon}\|_2^2=d$.

Let $t'$ be a copy of $t$ that is conditionally independent of $t$
given $\mathbf{Z}$, with the same conditional law. Then
$\mathbf{Y}=\boldsymbol{\eta}(\mathbf{Z},t)$ and
$\mathbf{Y}'=\boldsymbol{\eta}(\mathbf{Z},t')$ are i.i.d.\ given
$\mathbf{Z}$, so
\begin{equation}\label{eq:copy}
R_{TA}^{(d)}-R_{TD}^{(d)}
=\tfrac12\,\mathbb{E}\|\mathbf{Y}-\mathbf{Y}'\|_2^2 .
\end{equation}

Let $p_d=\mathbb{P}(\widehat{t}_d(\mathbf{Z})\neq t)$ and
$A_d=\{t\neq\widehat{t}_d(\mathbf{Z})\}\cup
\{t'\neq\widehat{t}_d(\mathbf{Z})\}$.
Outside $A_d$ we have $t=t'$, hence $\mathbf{Y}=\mathbf{Y}'$.
Since $(\mathbf{Z},t)$ and $(\mathbf{Z},t')$ have the same law,
\begin{equation}\label{eq:event}
\mathbb{P}(A_d)\le 2p_d .
\end{equation}

By conditional Jensen,
$\mathbb{E}\|\mathbf{Y}\|_2^4\le
\mathbb{E}\|\boldsymbol{\epsilon}\|_2^4=d(d+2)$, and the same holds
for $\mathbf{Y}'$. Using
$\|\mathbf{a}-\mathbf{b}\|_2^4\le
8(\|\mathbf{a}\|_2^4+\|\mathbf{b}\|_2^4)$,
\begin{equation}\label{eq:fourth}
\mathbb{E}\|\mathbf{Y}-\mathbf{Y}'\|_2^4\le 16\,d(d+2).
\end{equation}

By \eqref{eq:copy}, Cauchy--Schwarz, \eqref{eq:event} and
\eqref{eq:fourth},
\[
\begin{aligned}
R_{TA}^{(d)}-R_{TD}^{(d)}
&=\tfrac12\,\mathbb{E}\big[\|\mathbf{Y}-\mathbf{Y}'\|_2^2
\,\mathbf{1}_{A_d}\big]\\
&\le\tfrac12\sqrt{16\,d(d+2)}\,\sqrt{2p_d}
=2\sqrt{2d(d+2)\,p_d}.
\end{aligned}
\]

Finally, by independence and
Theorem~\ref{theorem:timestep_identifiability},
\[
p_d=\frac1T\sum_{i=1}^{T}
\mathbb{P}\big(\widehat{t}_d(\mathbf{x}_i^{(d)})\neq i\big)
\longrightarrow 0,
\]
because $T$ is fixed. Dividing the bound by $d$ gives
\[
0\le\frac{R_{TA}^{(d)}-R_{TD}^{(d)}}{d}
\le 2\sqrt{2\Big(1+\frac{2}{d}\Big)p_d}\longrightarrow 0 .
\]
\end{proof}
\end{document}